\documentclass[10pt]{article}

\usepackage[utf8]{inputenc}
\DeclareUnicodeCharacter{2212}{\ensuremath{-}}  
\usepackage[letterpaper,margin=1in,top=1in]{geometry}
\usepackage{times}
\usepackage{amsmath,amssymb,amsthm,bm}
\usepackage{booktabs}
\usepackage{graphicx}
\usepackage{subcaption}
\usepackage[font=small,labelfont=bf]{caption}
\usepackage{xcolor}
\usepackage{microtype}
\usepackage{mdframed}
\usepackage{float}
\usepackage{url}
\usepackage[round]{natbib}
\usepackage[colorlinks=true,linkcolor=blue!55!black,citecolor=teal!60!black,urlcolor=teal!60!black]{hyperref}
\usepackage{enumitem}

\definecolor{coactive}{HTML}{1F7895}

\newtheorem{theorem}{Theorem}
\newtheorem{proposition}{Proposition}

\theoremstyle{definition}
\newtheorem{assumption}{Assumption}
\newtheorem{definition}{Definition}
\theoremstyle{remark}

\newcounter{algctr}
\newenvironment{algobox}[1]{%
  \refstepcounter{algctr}\par\medskip
  \begin{mdframed}[linewidth=0.6pt,linecolor=black!55,backgroundcolor=black!2,
    innertopmargin=6pt,innerbottommargin=6pt]
  \noindent\textbf{Algorithm \thealgctr:} #1\par\vspace{2pt}\hrule\vspace{4pt}
  \small
}{\end{mdframed}\par\medskip}

\newcommand{\R}{\mathbb{R}}
\newcommand{\E}{\mathbb{E}}
\newcommand{\Hcal}{\mathcal{H}}
\newcommand{\Xcal}{\mathcal{X}}
\newcommand{\Tcal}{\mathcal{T}}
\newcommand{\cost}{\mathrm{cost}}

\newcommand{\Reg}{\mathcal{R}}
\newcommand{\eC}{\textsc{ec}\textsuperscript{2}}

\title{\vspace{-1.5em}\textbf{Cost-Aware Recovery-Pathway Identification and\\
Bayesian Optimization for Autonomous Materials Discovery}}

\author{
  Debajyoti Ray\thanks{The foundational results this paper builds on, the \eC{} algorithm
  for noisy Bayesian active learning \citep{golovin2010near} and GP-UCB \citep{srinivas2010gaussian},
  were developed by the authors at the California Institute of Technology.}\\
  \texttt{dray@coactive.science}\\
  Coactive Science
  \and
  Niranjan Srinivas\footnotemark[1]\\
  \texttt{niranjan@coactive.science}\\
  Coactive Science
}
\date{}

\begin{document}
\maketitle

\begin{abstract}
\noindent
Autonomous laboratories automate experimental execution, but a campaign must also decide which
recovery pathway merits optimization. We formulate this as a sequential decision problem with a
discrete pathway-identification stage and a continuous within-pathway optimization stage under
heterogeneous experimental costs. Our implementation, \textbf{Coactive learning}, combines a
cost-sensitive Bayesian hypothesis-discrimination policy motivated by \eC{} \citep{golovin2010near}
with Gaussian-process Bayesian optimization \citep{srinivas2010gaussian}. Under explicitly stated
assumptions, the expected spend of one fixed-budget campaign attempt is bounded by the expected
pathway-identification cost plus the capped within-pathway optimization budget. We evaluate on synthetic benchmarks constrained by
selected results reported for PNNL's CICERO selective-precipitation study
\citep{ritchhart2026agentic}; these are not reproductions of CICERO's laboratory campaigns. The
method performs comparably to an oracle-pathway Bayesian-optimization reference and to a strong
split-plate baseline that discriminates pathways with its first plate, without an oracle label for
the correct pathway (it is given a candidate hypothesis space and a diagnostic likelihood model). On
an NdFeB-inspired instance it avoids the simulated penalty of a commit-first baseline that initially
selects a plausible but inferior hydroxide pathway, a hypothetical wrong-first-commitment scenario
motivated by the hydroxide--oxalate performance contrast reported by CICERO. Sensitivity of the
conclusions to the assumed cost model is characterized. Code and benchmark are open source.
\end{abstract}

\section{Introduction}
Autonomous (self-driving) laboratories now automate the \emph{execution} of experiments:
liquid handlers dispense, characterization instruments stream results, and agentic planners tie the
loop together \citep{burger2020mobile,macleod2020self,abolhasani2023rise}. As execution is
automated, the remaining bottleneck moves to experiment \emph{selection}: deciding which
experiment, which protocol, and which scientific question merits instrument time. Unlike the data
regime of language models, where training corpora are effectively unlimited, experimental data
arrives one costly measurement at a time, so sample efficiency is the central concern and
\emph{active learning} the natural tool
\citep{settles2009active,mackay1992information,lindley1956measure}. This paper studies the
experiment-selection problem for an autonomous laboratory in which every action has a known price:
not only how to optimize a chosen protocol, but how to choose what to run.

In current platforms \citep{coley2019robotic,stach2021autonomous},
agents plan experiments, robotics execute them, instruments read the results back, and an
optimization policy proposes the next batch. A recent and especially relevant example is
\textbf{CICERO} \citep{ritchhart2026agentic}, an agentic workflow that recovers critical materials
from complex feedstocks through selective precipitation. CICERO combines feedstock characterization,
technoeconomic evaluation, hypothesis generation, experimental planning, robotic execution,
automated analysis, and Bayesian optimization (BO)
\citep{jones1998efficient,shahriari2016taking,frazier2018tutorial}; its results demonstrate the
feasibility of connecting agentic scientific reasoning with an autonomous experimental platform.

CICERO motivates two decision-theoretic questions that are not explicitly analyzed in its
published workflow. First, given several plausible recovery pathways, how should a campaign spend to
discriminate among them before committing continuous optimization budget to one? Committing
incorrectly wastes entire 96-well plates. Second, standard acquisition functions do not price
actions, so a \$100 diagnostic and a \$5{,}000 optimization plate cannot be weighed within one
budget. We study the resulting allocation problem: dividing a heterogeneous assay budget between
reducing uncertainty over candidate pathways and optimizing conditions within a selected pathway.

\paragraph{Contributions.} We formalize and benchmark this allocation problem.
\begin{itemize}[leftmargin=1.4em,itemsep=1pt,topsep=2pt]
\item \textbf{A two-stage cost-sensitive policy (\S\ref{sec:framework}).} Coactive
  learning\footnote{We use \emph{Coactive learning} as a contraction of \emph{cost-aware active
  learning} (and a nod to its origin); it is distinct from the preference-feedback
  ``coactive learning'' of \citet{shivaswamy2012online}.} is a two-stage cost-sensitive policy: a pathway-discrimination stage that selects priced
  diagnostics by posterior decision-region impurity reduction per dollar, motivated by \eC{}
  \citep{golovin2010near}, followed by batch Bayesian optimization with GP-UCB
  \citep{srinivas2010gaussian} within the selected pathway, from a center-biased initial design.
\item \textbf{A conditional cost analysis (\S\ref{sec:theory}).} Under explicit assumptions, the
  expected spend of one fixed-budget attempt is bounded by the expected pathway-identification cost
  plus the capped optimization budget (Prop.~\ref{thm:composite}), and a commit-first baseline that
  initially selects an incorrect pathway pays a quantifiable wrong-commitment penalty
  (Prop.~\ref{prop:dominance}).
\item \textbf{A CICERO-inspired synthetic benchmark and open implementation (\S\ref{sec:benchmark}).}
  We construct synthetic response surfaces constrained by selected endpoints and qualitative trends
  reported by CICERO, for three feedstock-inspired instances, and release a dependency-light
  reference implementation with tests.
\item \textbf{Empirical evaluation and ablations (\S\ref{sec:experiments}).} The method performs
  comparably (within $2\%$ on the NdFeB-inspired instance) to an oracle-pathway BO reference and to
  a split-plate baseline, without an oracle pathway label, and avoids the simulated
  wrong-commitment penalty of a commit-first baseline ($\approx\!1/5$ its cost on that instance);
  ablations and sensitivity studies characterize when priced discrimination is justified.
\end{itemize}
We emphasize at the outset (\S\ref{sec:benchmark}, \S\ref{sec:discussion}) that our experiments use a
\emph{endpoint-constrained synthetic simulator}, not new wet-lab runs; the contribution is
the decision policy and its analysis, isolated on a benchmark that reproduces the paper's reported
optima.

\paragraph{Outlook.} The framework is not specific to the CICERO workflow. The two-phase structure
(identify the decision region, then optimize within it, each stage priced explicitly) applies to
any campaign that first chooses among discrete protocol families and then tunes continuous
parameters under heterogeneous assay costs, including separations, synthesis-route selection, and
formulation problems such as battery electrolytes. The method takes a hypothesis space and priced
actions as input and returns experiment selections, so it can be used inside existing agentic
platforms; \S\ref{sec:future} discusses this direction.

\section{Problem Formulation}\label{sec:formulation}
A recovery \emph{campaign} on a feedstock has two coupled unknowns. \emph{Which} chemistry to use, and
\emph{how} to tune it.

\paragraph{Decision regions and diagnostics (discovery).}
Let $\Hcal$ be a finite set of hypotheses about the feedstock's recovery mechanism, with prior
$p_0$ over $\Hcal$. A known map $r:\Hcal\to\Reg$ assigns each hypothesis to a \emph{decision region}: a set of
hypotheses that imply the same actionable \emph{recovery pathway}, here a target metal and
precipitating reagent. Hypotheses sharing a region are equivalent for the purpose of acting; we use
``recovery pathway'' in scientific prose and ``decision region'' in formal statements. We have a set of diagnostics $\Tcal$ (e.g.\ an ICP-MS composition profile, a small
bench precipitation probe, a technoeconomic screen), each with cost $c(t)>0$ and a known noisy
outcome model $p(o\mid h,t)$. Running $t$ and observing $o$ updates the posterior over $\Hcal$ by
Bayes' rule. The discovery goal is to identify the correct region $r^\star=r(h^\star)$ with posterior
confidence $\ge 1-\delta$, at minimum expected cost.

\paragraph{Response surface (optimization).}
Conditioning on a committed (estimated) region $\widehat r$, recovery quality is an unknown function
$g:\Xcal\to\R$ (purity, yield, or separation factor) over a compact domain $\Xcal\subset\R^d$ of
controllable knobs (reagent equivalents, pH, concentration). Evaluations are noisy,
$y=g(x)+\eta$ with $\eta$ sub-Gaussian, and are run in \emph{batches} (a 96-well plate evaluates
$q$ distinct conditions in parallel) at cost $c_{\mathrm{opt}}$ per plate. We seek a condition that is
$\varepsilon$-optimal, $g(x)\ge g^\star-\varepsilon$ with $g^\star=\max_x g(x)$, equivalently, when
$\tau=g^\star-\varepsilon$, that clears an application target $\tau$. We assume a continuous
kernel on compact $\Xcal$, so $g^\star$ and acquisition maximizers exist.

\paragraph{Objective.}
The campaign goal is to identify $r^\star$ with confidence $1-\delta$ and return an
$\varepsilon$-optimal condition within it. For one campaign attempt, define the expected spend as
\begin{equation}
\textstyle
\E\big[\cost\big] \;=\; \E\Big[\sum_{t\in\text{diagnostics}} c(t)\Big]
\;+\; c_{\mathrm{opt}}\,\E\big[\#\{\text{plates}\}\big].
\label{eq:objective}
\end{equation}
This common cost unit makes diagnostic and optimization expenditures comparable. Within each stage,
the policy divides its stage-specific expected utility by action cost; it does not optimize a single
cross-stage cost-to-target acquisition.

\section{The Coactive-Learning Framework}\label{sec:framework}
Coactive learning runs one loop with two phases sharing a cost-normalized criterion.

\subsection{Discovery: cost-aware \eC{}}\label{sec:disc}
We adopt the \eC{} (Equivalence Class Edge Cutting) objective of \citet{golovin2010near}. Define a
weighted graph on $\Hcal$ whose edges connect hypotheses in \emph{different} regions, with edge weight
equal to the product of their probabilities. Formally, for a posterior $p$ over $\Hcal$ the total
cross-region edge weight is
\begin{equation}
W(p) \;=\; \tfrac12\Big(\big(\textstyle\sum_{h} p(h)\big)^2 - \sum_{\rho\in\Reg}\big(\textstyle\sum_{h: r(h)=\rho} p(h)\big)^2\Big),
\label{eq:edge}
\end{equation}
and the \eC{} objective accrued after observing a set of (diagnostic, outcome) pairs is the expected
cross-region weight \emph{cut}: after a history $\psi$ the realized accrued cut is
$W(p_0)-Z_\psi^2W(p_\psi)$, and expectation enters only through the conditional marginal
acquisition. Driving the cut to its quota $W(p_0)$ identifies the true region only under an
identifiability condition (every complete outcome vector compatible with a single region) that we do
not assume; see Appendix~\ref{app:bounds}.
Exact \eC{} \citep{golovin2010near} places fixed \emph{prior-weighted} edges between hypotheses
in different regions and cuts an edge when an observation is inconsistent with either endpoint; after
a history $\psi$ with prior mass $Z_\psi$ and normalized posterior $p_\psi$, the remaining edge
mass is $Z_\psi^2\,W(p_\psi)$, not $W(p_\psi)$. Under the finite deterministic-realization
expansion this objective is adaptive submodular and strongly adaptively monotone even with correlated
noise. Our implementation uses a distinct surrogate: it evaluates $W$ at the \emph{current
posterior}, so the acquisition is expected reduction in posterior cross-region mass (posterior
decision-region impurity), not the prior-weighted edge count of exact \eC{}. We therefore
present it as a practical acquisition motivated by \eC{} rather than an algorithm inheriting its
guarantee (\S\ref{sec:theory}). At each step the policy selects the unused diagnostic maximizing
expected impurity reduction \emph{per dollar},
\begin{equation}
t^\star \;=\; \arg\max_{t\in\Tcal}\ \frac{\E_{o}\big[\,W(p)-W(p\mid t,o)\,\big]}{c(t)},
\label{eq:ec2greedy}
\end{equation}
and stops once one region's posterior mass exceeds a threshold $1-\delta_{\mathrm{disc}}$ (we use
$0.9$). This posterior-threshold stopping targets approximate Bayesian identification,
$\Pr(r^\star \neq \widehat r) \le \delta_{\mathrm{disc}}$ under the model, rather than the exact
edge cover analyzed for \eC{}.

\subsection{Optimization: batch GP-UCB}\label{sec:opt}
Within the committed region we place a Gaussian-process prior on $g$ \citep{rasmussen2006gaussian} and
run GP-UCB \citep{srinivas2010gaussian}: at round $t$ with posterior mean $\mu_t$ and standard
deviation $\sigma_t$, the acquisition is $\alpha^{\mathrm{UCB}}_t(x)=\mu_t(x)+\beta_t^{1/2}\sigma_t(x)$
with the confidence schedule $\beta_t=2\log\!\big(|D|\,t^2\pi^2/6\delta\big)$ for a discretized domain
$D$ (Thm.~1 of \citealp{srinivas2010gaussian}). Plates of $q$ conditions are formed by sequential
greedy maximization with ``kriging-believer'' fantasies \citep{ginsbourger2010kriging,gonzalez2016batch,desautels2014parallelizing};
we use LogEI \citep{ament2023unexpected} and Thompson sampling \citep{russo2018tutorial} as drop-in
alternatives, and track a purity/yield/cost Pareto front for multi-objective campaigns via
hypervolume improvement \citep{daulton2021parallel}.

\subsection{Stage structure, pricing, and initialization}\label{sec:unify}
Within each stage, candidate actions are scored per unit cost: diagnostics by expected impurity
reduction per dollar \eqref{eq:ec2greedy}, and optimization batches by their acquisition value per
dollar,
\begin{equation}
a^{\$}(\cdot) \;=\; \frac{\text{expected utility}(\cdot)}{\cost(\cdot)} .
\label{eq:perdollar}
\end{equation}
These two per-dollar scores have different units (posterior mass versus objective value), so they
are not compared against each other across stages: the policy is an explicitly \emph{two-stage}
procedure that discriminates pathways first and optimizes second (Algorithm~1). Defining a single
common-currency acquisition across stages (for example, expected reduction in total remaining
cost-to-target) is an open design problem we do not claim to solve here. Finally, the first
optimization plate uses a \emph{center-biased initial design}, a mild bias toward the domain centre
with no information about the optimum and no use of data from other campaigns
(Appendix~\ref{app:sim}).

\begin{algobox}{Coactive learning: cost-aware active-learning campaign}
\textbf{Input:} hypotheses $\Hcal$, prior $p_0$, diagnostics $\Tcal$ with costs, confidence $1-\delta$,
target $\tau$, plate size $q$, initial-design distribution $\pi$.\\[2pt]
\emph{Discovery.} \textbf{while} no region has posterior mass $\ge 1-\delta$ \textbf{and} unused diagnostics remain:\\
\hspace*{1.2em} run diagnostic $t^\star$ maximizing expected posterior impurity reduction per dollar \eqref{eq:ec2greedy}; update posterior.\\
\textbf{if} no region reached the threshold: \textbf{return} \emph{abstain}.\\
\emph{Commit} region $\widehat r=\arg\max_\rho p(\rho)$ and its target/reagent.\\
\emph{Optimization.} seed plate $1$ from $\pi$; \textbf{for} $t=2,3,\dots$:\\
\hspace*{1.2em} fit GP to all observations in $\widehat r$; design a $q$-condition plate by batch GP-UCB;\\
\hspace*{1.2em} run plate; \textbf{if} best observed $\ge\tau$ and at least one plate has followed a first success: \textbf{return}.
\end{algobox}

Setting the discrimination stage, cost-normalization, and center-biased initialization aside
recovers a standard single-pathway BO loop; these three ingredients are what the policy adds, and \S\ref{sec:experiments}
ablates them.

\section{Theoretical Analysis}\label{sec:theory}
We bound the expected spend of a fixed-budget campaign attempt (the expectation of
\eqref{eq:objective} with a random plate count, $\E[\#\{\text{plates}\}]$). Proofs are in Appendix~\ref{app:proofs}. The analysis combines one classical optimization guarantee with assumed discrimination-stage properties, yielding a fixed-budget spend bound and a success probability.
We first recall the structural property the discovery phase relies on.

\begin{definition}[Adaptive submodularity; \citealp{golovin2011adaptive}]\label{def:adsub}
Let $f$ map partial realizations (sets of observed (item, outcome) pairs) to $\R_{\ge0}$ and let
$p$ be a prior over realizations $\Phi$. $f$ is \emph{adaptive submodular} w.r.t.\ $p$ if for all
partial realizations $\psi\subseteq\psi'$ and every item
$e\notin\operatorname{dom}(\psi')$,
$\Delta(e\mid\psi')\le\Delta(e\mid\psi)$, where
$\Delta(e\mid\psi)=\E\big[f(\psi\cup\{(e,\Phi(e))\})-f(\psi)\,\big|\,\psi\big]$.
It is \emph{strongly adaptively monotone} if $f(\psi\cup\{(e,o)\})\ge f(\psi)$ for every
feasible outcome $o$. Cover guarantees additionally require coverability, an $\eta$-gap, and
cost/normalization hypotheses specific to each theorem; adaptive submodularity alone is insufficient.
\end{definition}

\begin{assumption}[Discrimination stage]\label{ass:disc}
The discrimination policy of \S\ref{sec:disc}, run with posterior threshold
$1-\delta_{\mathrm{disc}}$ under a correctly specified likelihood model, either identifies a decision region or, if all diagnostics are exhausted below threshold,
\emph{abstains}; conditional on committing, the committed region $\widehat r$ equals the true region
$r^\star$ with probability at least $1-\delta_{\mathrm{disc}}$, and expected diagnostic cost is at
most $C_{\mathrm{disc}}$. Counting abstention as failure, the discrimination-stage failure
probability is at most $\delta_{\mathrm{disc}}+\Pr(\text{abstention})$. We do \emph{not} assume a competitive ratio for
$C_{\mathrm{disc}}$: the missing justification is that the adaptive-cover hypotheses (objective,
monotonicity, coverability, $\eta$-gap, cost normalization) have not been established for the
implemented posterior-impurity surrogate and threshold stopping rule. Outcomes correlated through
the latent hypothesis preclude the independent-item-state result of \citet{hellerstein2021tight}
specifically, not all adaptive-cover results (Appendix~\ref{app:bounds}); the known bounds are
cited as motivation rather than invoked.
\end{assumption}

\begin{assumption}[Optimization regularity]\label{ass:opt}
$g$ lies in the RKHS of a known kernel $k$ with $\|g\|_k\le B$, $k(x,x)\le1$, and observation noise is
conditionally $R$-sub-Gaussian. Let $\gamma_T$ denote the maximum information gain of $k$ after
$T$ evaluations.
\end{assumption}

For exact \eC{} on equivalence-class determination, greedy cost-sensitive selection is
near-optimal for the edge-cover objective; the history and assumptions of that bound are reviewed in
Appendix~\ref{app:bounds}. Because our implementation optimizes posterior impurity with threshold
stopping, we do not claim that competitive ratio here, and instead analyze the campaign
conditionally on the discrimination stage's error and cost (Assumption~\ref{ass:disc}).

\begin{theorem}[Optimization budget to $\varepsilon$; idealized sequential procedure]\label{thm:opt}
Under Assumption~\ref{ass:opt}, run sequential IGP-UCB
\citep{chowdhury2017kernelized} in the committed region with width
$\sqrt{\beta_t}=B+R\sqrt{2(\gamma_{t-1}+1+\ln(1/\delta))}$ and the matching posterior
regularization. With probability $\ge1-\delta$ the cumulative regret satisfies
$R_T=O\big(\sqrt{T}(B\sqrt{\gamma_T}+\gamma_T)\big)$. Consequently some \emph{queried} point is
$\varepsilon$-optimal once $R_T/T\le\varepsilon$; the sufficient condition $\beta_T\gamma_T=o(T)$ (satisfied by
the squared-exponential kernel, where $\gamma_T=O((\log T)^{d+1})$ and
$T_\varepsilon=\tilde O(\varepsilon^{-2})$) ensures a finite $T_\varepsilon$ exists for any
fixed $\varepsilon$; the condition is sufficient, not necessary. The guarantee concerns the existence of an $\varepsilon$-optimal queried
point; identifying or returning it from noisy observations is a separate estimation step not covered
by the bound.
\end{theorem}
\noindent Translating this sequential evaluation bound into parallel plate rounds requires a
specifically analyzed batch policy \citep{desautels2014parallelizing,contal2013parallel}; we do not
claim that conversion for the kriging-believer batches used in the benchmark, where full-plate costs
enter only the empirical cost model.
\noindent This analysis applies to an idealized fixed-kernel sequential procedure with a known RKHS
norm bound and unclipped sub-Gaussian noise; the benchmark's clipping of purity observations at
$100\%$ introduces a residual bias outside the theorem. The benchmark implementation instead learns hyperparameters by marginal likelihood, optimizes
over a candidate grid, forms batches with kriging-believer fantasies, and stops on noisy observed
values (Appendix~\ref{app:sim}); it is motivated by, but not identical to, the algorithm covered by
the theorem, and we do not claim the bound characterizes the implementation.

\begin{proposition}[Fixed-budget spend and success bound]\label{thm:composite}
Suppose the discrimination stage satisfies Assumption~\ref{ass:disc}: it may abstain, its expected
diagnostic cost is at most $C_{\mathrm{disc}}$, and \emph{conditional on committing} the committed
region is correct with probability at least $1-\delta_{\mathrm{disc}}$. Suppose the within-region
optimization stage is run for a fixed budget of at most $T$ evaluations on every branch that commits
to a region, with no optimization after abstention; this cap holds unconditionally. Separately,
conditional on correct commitment, the optimization stage returns an $\varepsilon$-optimal
condition with probability at least $1-\delta_{\mathrm{opt}}$. Then the expected spend of this single fixed-budget attempt is at most
\begin{equation}
C_{\mathrm{disc}} \;+\; C_{\mathrm{opt}}(T), \qquad
C_{\mathrm{opt}}(T) = c_{\mathrm{opt}}\big\lceil T/q\big\rceil
\end{equation}
under the full-plate cost model, and the probability of returning an $\varepsilon$-optimal
condition is at least $1-\delta_{\mathrm{disc}}-\Pr(\text{abstention})-\delta_{\mathrm{opt}}$.
(Spend is controlled on every branch because the optimization budget is capped unconditionally and
abstention incurs no optimization spend.)
\end{proposition}
\noindent\emph{Proof.} Stage spends add, so expected spend is bounded by linearity of expectation;
success requires both stages to succeed, and a union bound over the two failure events gives the
probability statement. $\square$

The next result quantifies the cost of committing to a pathway incorrectly under a specific
restricted baseline.

\begin{proposition}[Penalty of wrong initial commitment; restricted baseline]\label{prop:dominance}
Consider a \emph{commit-first} baseline that selects one decision region $\widehat r$ without
diagnostics (a fixed default, or an independent draw from $p_0$), optimizes within it, and, if
$\widehat r\neq r^\star$, expends exactly $T_{\mathrm{fail}}$ evaluations
($c_{\mathrm{opt}}\lceil T_{\mathrm{fail}}/q\rceil$ in plates) before detecting failure and
switching to $r^\star$. For each policy let $C_{\mathrm{post}}$ denote its \emph{residual} optimization spend, defined on
every branch (zero after abstention), excluding diagnostics and excluding the separately counted
wrong-commitment spend $c_{\mathrm{opt}}\lceil T_{\mathrm{fail}}/q\rceil$. Assume
$\E[C_{\mathrm{post}}^{\mathrm{commit\text{-}first}}]\ge\E[C_{\mathrm{post}}^{\mathrm{two\text{-}stage}}]$. Let
$p_{\mathrm{corr}}=\Pr(\widehat r=r^\star)$; when the prior induces a uniform distribution over
decision regions, $p_{\mathrm{corr}}=1/|\Reg|$ for either selection rule. Then
\begin{equation}
\E[\cost_{\mathrm{commit\text{-}first}}]-\E[\cost_{\mathrm{two\text{-}stage}}]
\;\ge\;(1-p_{\mathrm{corr}})\,c_{\mathrm{opt}}\big\lceil T_{\mathrm{fail}}/q\big\rceil
\;-\;C_{\mathrm{disc}} .
\end{equation}
A \emph{sufficient} condition for diagnostics to be beneficial against this baseline is
$C_{\mathrm{disc}}<(1-p_{\mathrm{corr}})\,c_{\mathrm{opt}}\lceil T_{\mathrm{fail}}/q\rceil$.
The two policies' success probabilities differ in general and are not equated here. The bound
applies to this baseline only, not to split-plate, mixed-space, or portfolio strategies
(\S\ref{sec:experiments}).
\end{proposition}

\section{A CICERO-Inspired Synthetic Benchmark}\label{sec:benchmark}

\paragraph{The CICERO campaigns.} The CICERO study comprises three closed-loop campaigns on real
feedstocks, each an alternating sequence of decisions and priced measurements. A campaign proceeds
as: (i) \emph{Diagnose}, a quantitative ICP-MS reading of
the feedstock that returns element identities and concentrations (for the produced water: Ca
$>\!13{,}000$\,ppm, Na $>\!6{,}000$, Mg $1{,}840$, Sr $450$; for the NdFeB leachate: Fe $62{,}000$,
Nd $23{,}800$, Pr $5{,}900$, B $920$, Co $470$, Dy $27$, Tb $16$); (ii) \emph{Evaluate/Hypothesize}, in which characterization informs economic evaluation and
literature-grounded pathway hypotheses, from which the campaign's pathway is selected; and (iii)
\emph{Optimize}, rounds of $96$-well plates ($8$ conditions $\times\,12$ replicates) with commodity
reagents (NaOH, NaHCO$_3$, sodium oxalate), read out by automated ICP-MS as purity (mol\%), yield,
and separation factor. A single data point is therefore a (condition, assay) pair inside a sequential,
budget-constrained loop; the object of study is the campaign-level policy rather than any single
model fit.

\paragraph{Benchmark construction.} We construct synthetic response surfaces constrained to
reproduce selected endpoints and qualitative trends reported by CICERO \citep{ritchhart2026agentic}
(listed below), with Gaussian assay noise whose level ($4.5\%$ relative) is an author assumption
motivated by the $4.0$--$7.6\%$ repeatability range CICERO reports for its automated ICP-MS pipeline
in one produced-water experiment. The raw campaign trajectories and complete response measurements
are not publicly available, and the equations (Appendix~\ref{app:sim}) are author-designed
benchmark functions. These instances should therefore be interpreted as controlled benchmarks
inspired by CICERO, not as statistical reconstructions or performance evaluations of CICERO's
laboratory campaigns; absolute costs and cost ratios depend on the modeling assumptions, whose
provenance and sensitivity are reported in Appendix~\ref{app:sim}.

\begin{itemize}[leftmargin=1.4em,itemsep=1pt,topsep=2pt]
\item \textbf{Produced water (Mg).} One knob (NaOH equivalents); Mg(OH)$_2$ purity peaks near
  $99\%$ at $\approx\!3.4$ equivalents with $\approx\!86\%$ yield, falling as Ca co-precipitates at
  higher equivalents.
\item \textbf{SmCo leachate (Sm).} Two knobs (NaHCO$_3$ $\times$ NaOH); Sm purity peaks near $92\%$.
\item \textbf{NdFeB leachate (Fe/REE).} The discovery decision is sharp: oxalate yields a
  Fe/REE separation factor $\approx\!200$ across a wide concentration range, whereas hydroxide gives a
  low separation factor and collapses at high concentration. A second stage splits Nd/Pr
  (separation factor up to $\approx\!1.35$).
\end{itemize}
Each feedstock defines an \eC{} discovery problem over candidate (target, reagent) hypotheses with
priced diagnostics, and an optimization stage over the knobs above. We charge \$50 per ICP-MS well; an
optimization plate is $8$ conditions $\times\,12$ replicates ($=\!\$4{,}800$). This matches the
plate structure CICERO reports for its produced-water reproducibility experiment; for benchmarking
consistency we impose this common synthetic batch structure on all instances, which differs from
CICERO's campaign-specific designs. Plates are charged in full when submitted: cost accrues per
whole plate, never per well within a submitted plate. A pathway-committed agent still characterises the feedstock to choose a target,
but this is a single ICP-MS reading (as in CICERO), \emph{not} a full plate; \eC{} diagnostics are
priced comparably (\$60--\$120). ``Cost-to-target'' is the spend to first clear the application
target; we also report the best \emph{true} (noiseless) value reached against the numerically computed
achievable optimum of each surface.

\section{Experiments}\label{sec:experiments}
We compare against four references on identical surrogates and acquisition: \textbf{Oracle-pathway
BO}, given the correct pathway in advance (an upper reference); \textbf{Split-plate BO}, which
divides its first plate across the candidate reagent pathways, commits to the best-observed pathway,
and then optimizes within it; \textbf{Commit-first BO}, a simulated baseline that commits to a
plausible default pathway and backtracks on failure (this trajectory is a benchmark construction
motivated by the hydroxide--oxalate contrast CICERO reports, not an observed CICERO campaign); and
\textbf{Random search}. The method is run without an oracle pathway label but with the specified
hypothesis space and diagnostic likelihood model. Swapping the baseline acquisition between GP-UCB
and LogEI changes results negligibly ($90.4$ vs.\ $90.3$ best true value on SmCo). All numbers are
means $\pm$ standard deviation over $20$ seeds; costs are full-plate totals per seed (a mean such as
\$7{,}600 averages seeds that finished in one plate with seeds that needed two). Full setup in
Appendix~\ref{app:sim}.

\begin{table}[h]
\centering
\caption{Cost-to-target and best \emph{true} value reached (mean over 20 seeds), against the achievable
optimum (shown in each instance label). Coactive learning performs comparably to Oracle-pathway BO
and Split-plate BO on all
three feedstocks; on NdFeB, where the chemistry choice is non-obvious, it reaches target at a fraction
of the cost of the simulated commit-first baseline, which commits to hydroxide and must backtrack to
oxalate. On produced water and SmCo the expert and default baselines coincide because the target is
obvious.}
\label{tab:main}
\small
\begin{tabular}{llcc}
\toprule
Feedstock & Method & Cost-to-target (\$) & Best achieved (true) \\
\midrule
Produced water (Mg) (opt 99.2) & \textbf{Coactive learning} & \textbf{4,980 $\pm$ 0} & \textbf{99.2} \\
 & Oracle-pathway BO & 5,370 $\pm$ 1,440 & 99.1 \\
 & Split-plate BO & 5,370 $\pm$ 1,440 & 99.1 \\
 & Commit-first BO & 5,370 $\pm$ 1,440 & 99.1 \\
 & Random search & 5,130 $\pm$ 1,046 & 99.1 \\
\midrule
SmCo (Sm) (opt 92.0) & \textbf{Coactive learning} & \textbf{7,600 $\pm$ 4,150} & \textbf{88.9} \\
 & Oracle-pathway BO & 9,450 $\pm$ 3,552 & 90.4 \\
 & Split-plate BO & 9,450 $\pm$ 3,552 & 90.4 \\
 & Commit-first BO & 9,450 $\pm$ 3,552 & 90.4 \\
 & Random search & 10,170 $\pm$ 5,011 & 86.8 \\
\midrule
NdFeB (Fe/REE) (opt 216.0) & \textbf{Coactive learning} & \textbf{4,985 $\pm$ 0} & \textbf{215.9} \\
 & Oracle-pathway BO & 4,890 $\pm$ 0 & 216.0 \\
 & Split-plate BO & 4,890 $\pm$ 0 & 215.9 \\
 & Commit-first BO & 24,090 $\pm$ 0 & 215.9 \\
 & Random search & 24,090 $\pm$ 0 & 215.2 \\
\bottomrule
\end{tabular}
\end{table}

\begin{figure}[t]
\centering
\includegraphics[width=\textwidth]{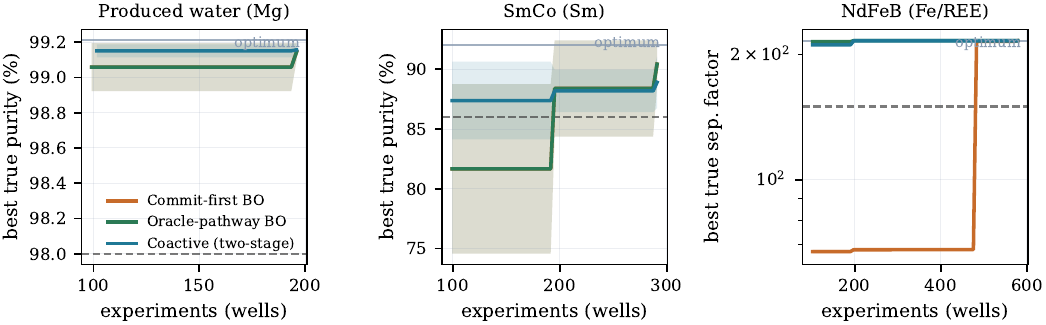}
\caption{Best \emph{true} (noiseless) recovery value vs.\ number of experiments (mean $\pm$ std over 20
seeds). Grey line: achievable optimum; dashed: application target. Coactive learning (teal) tracks the
oracle-pathway reference (green); the commit-first baseline (copper) coincides with it on produced
water and SmCo but, on the NdFeB-inspired instance (log scale), remains on hydroxide for
$\sim\!480$ wells before backtracking to oxalate (a simulated trajectory; see \S\ref{sec:experiments}).}
\label{fig:conv}
\end{figure}

\paragraph{Main result (Table~\ref{tab:main}, Fig.~\ref{fig:conv}).}
On every instance the method performs comparably to Oracle-pathway BO without an oracle pathway
label: on the NdFeB-inspired instance, \$4{,}985 versus \$4{,}890 to target (within $2\%$). The
commit-first baseline that initially selects hydroxide pays the wrong-commitment penalty of
Proposition~\ref{prop:dominance}, about \$24{,}090 versus \$4{,}985 ($\approx\!4.8\times$);
this trajectory is a simulated construction motivated by the hydroxide--oxalate performance contrast
CICERO reports, not CICERO's observed behavior. Importantly, the split-plate baseline also avoids
that penalty at essentially the same cost (\$4{,}890): with only two testable pathways and a broad
above-target basin on the correct pathway, one plate divided across pathways both discriminates and
optimizes. On this benchmark, priced diagnostics and a split first plate are therefore equally
effective ways to buy pathway information; the practical case for diagnostics-first arises when the
candidate pathway set is too large for a single plate to cover at useful per-pathway resolution,
when diagnostics are much cheaper than wells, or when target basins are narrow. We state this as a
hypothesis to be tested at larger pathway counts, not a demonstrated result. On produced water and
SmCo the pathway is unambiguous, all pathway-informed baselines coincide, and discrimination is a
small net overhead; on the one-dimensional produced-water instance random search is competitive.

\begin{figure}[t]
\centering
\begin{subfigure}[t]{0.40\textwidth}
  \includegraphics[width=\textwidth]{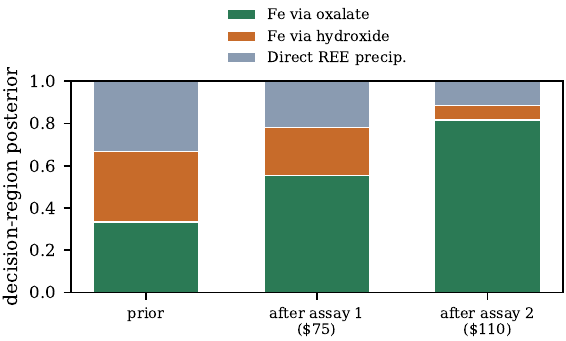}
  \caption{\eC{} decision-region posterior over two cheap NdFeB diagnostics (\$185 total).}
  \label{fig:ec2}
\end{subfigure}\hfill
\begin{subfigure}[t]{0.56\textwidth}
  \includegraphics[width=\textwidth]{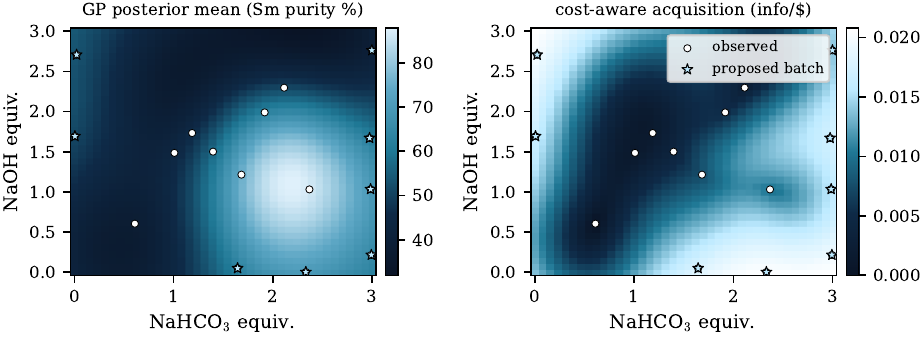}
  \caption{SmCo: GP posterior mean and cost-normalized acquisition over the two reagent knobs, with
  observed (circles) and proposed (stars) conditions.}
  \label{fig:surr}
\end{subfigure}
\caption{Discovery (left) concentrates the posterior on the correct chemistry before any optimization
plate; optimization (right) then designs plates by cost-normalized GP-UCB.}
\label{fig:qual}
\end{figure}

\paragraph{Ablations (Table~\ref{tab:abl}).}
Removing the discrimination stage sends the NdFeB cost to the commit-first figure
($\approx\!\$24{,}000$), confirming pathway information as the dominant lever. The center-biased
initial design gives a modest benefit on the two-dimensional SmCo surface and none on NdFeB.
Removing cost-normalization inside the optimization stage changes nothing: with uniform plate costs
the per-dollar division is a constant rescaling of the acquisition, so it is inert there by
construction; prices matter in this benchmark only through the discrimination stage, and would
enter the optimization stage only under heterogeneous per-action costs.

\begin{table}[h]
\centering
\caption{Ablation: cost-to-target (\$, mean over 20 seeds) as each Coactive learning ingredient is
removed. Discovery is the dominant lever; the default-agent row is shown for reference.}
\label{tab:abl}
\small
\begin{tabular}{lccc}
\toprule
Method & Produced water (Mg) & SmCo (Sm) & NdFeB (Fe/REE) \\
\cmidrule(lr){2-4}
 & \multicolumn{3}{c}{Cost-to-target (\$)} \\
\midrule
\textbf{Coactive learning} & \textbf{4,980} & \textbf{7,600} & \textbf{4,985} \\
w/o discrimination & 4,890 & 8,490 & 24,090 \\
w/o cost-aware & 4,980 & 7,600 & 4,985 \\
w/o center-biased init & 5,220 & 10,000 & 4,985 \\
Commit-first BO & 5,370 & 9,450 & 24,090 \\
\bottomrule
\end{tabular}
\end{table}

\paragraph{Interpreting the baseline comparisons.}
Two distinctions in Table~\ref{tab:main} explain why the baselines sometimes perform no better than
random search. The
\emph{oracle vs.\ commit-first} contrast isolates the pathway decision: the two coincide on
produced water and SmCo, where the pathway is unambiguous, and diverge only on
the NdFeB-inspired instance, where the plausible default reagent (hydroxide) is inferior. The \emph{BO vs.\
random} contrast isolates \emph{optimizer quality}, and here the apparent tie is an
artifact of the metric, not a weak optimizer. Cost-to-target measures the spend to
first clear a fixed application target; on low-dimensional surfaces with a broad
near-optimal basin, a single space-filling plate of $96$ wells already contains a
point above a modest target with high probability, so time-to-target saturates and
any reasonable design looks similar. The optimizer's advantage appears when we instead
measure progress toward the \emph{optimum}, or tighten the target. In an isolated
optimization study on SmCo (no discovery, no warm-start, identical initial plate;
Appendix~\ref{app:sim}), GP-UCB reaches a best true purity of $91.9\%$ versus random's
$89.0\%$ by the fifth plate, and the gap in plates-to-target widens monotonically as
the target hardens ($1.9$ vs.\ $2.7$ plates at $86\%$; $2.8$ vs.\ $3.7$ at $91.5\%$,
where random reaches target in only $45\%$ of runs). On these instances the cost-to-target differences are therefore governed by the pathway decision
rather than by the acquisition function, on which all BO-based methods agree.

\section{Related Work}\label{sec:related}
\paragraph{Noisy Bayesian active learning.} Identifying a hypothesis from costly, noisy tests has a
long history, from generalized binary search \citep{dasgupta2004analysis,nowak2011geometry} to the
\eC{} algorithm and adaptive submodularity \citep{golovin2010near,golovin2011adaptive}, whose cover
guarantees were subsequently corrected and are theorem-specific: independent item states for
\citet{hellerstein2021tight}, correlated realizations with unit costs for
\citet{esfandiari2021adaptivity}, and nonuniform costs for \citet{cui2023greedy}
(Appendix~\ref{app:bounds}). \textsc{eced} \citep{chen2017near} provides a related approach for noisy and correlated testing
settings; we do not implement it. Classical
experimental design \citep{lindley1956measure,mackay1992information} and submodular sensing
\citep{krause2008near} share the information-theoretic view. We use these for the \emph{discovery}
phase and, crucially, in a cost-normalized form.

\paragraph{Bayesian optimization.} GP-UCB and its regret theory
\citep{srinivas2010gaussian,srinivas2012information} underpin our optimization phase; we also draw on
EI/LogEI \citep{jones1998efficient,ament2023unexpected}, Thompson sampling \citep{russo2018tutorial},
batch methods \citep{ginsbourger2010kriging,gonzalez2016batch,desautels2014parallelizing,wilson2018maximizing},
and multi-objective BO \citep{daulton2021parallel}, as surveyed by
\citet{shahriari2016taking,frazier2018tutorial} and implemented in tools such as
BoTorch \citep{balandat2020botorch}. \emph{Cost-aware} BO weighs evaluation cost
\citep{snoek2012practical,astudillo2021multistep}; we extend the idea across the
discovery/optimization boundary, normalizing a \emph{decision} acquisition and an optimization
acquisition by a common cost. Closest in spirit, a recent line treats active learning and Bayesian
optimization from a unified standpoint \citep{difiore2023active} and pairs GP-UCB with hypothesis
generation (PA-GP-UCB; \citealp{pagpucb2026}); we differ in coupling \emph{equivalence-class
decision-region} identification with within-pathway Bayesian optimization in an explicitly two-stage
policy, with a
conditional fixed-budget cost analysis and an autonomous-lab benchmark.

\paragraph{Autonomous labs and materials discovery.} Self-driving laboratories
\citep{burger2020mobile,macleod2020self,coley2019robotic,stach2021autonomous,abolhasani2023rise} and
BO-for-chemistry systems \citep{hase2018phoenics,lookman2019active} established the closed loop; CICERO
\citep{ritchhart2026agentic} applied it to critical-materials recovery with vanilla BO. Our
contribution is orthogonal and composable: a \emph{policy} for that loop with such a conditional
cost analysis, one that decides what chemistry to run before optimizing it.

\section{Discussion and Limitations}\label{sec:discussion}
The comparison is designed to be fair: the baseline shares the same GP and acquisition, differing
only in the three ingredients we add, so the measured gains are attributable to the policy, not to
handicapping. The main limitations are as follows.

\begin{itemize}[leftmargin=1.4em,itemsep=1pt,topsep=2pt]
\item \textbf{In-silico evaluation on author-designed surfaces.} Response surfaces are synthetic
  functions constrained by selected published endpoints, not fits to raw measurements; absolute
  numbers are benchmark results, not laboratory claims, and the benchmark does not reproduce
  CICERO's measured costs or trajectories. Validation against real campaign logs (replaying the
  policy against the recorded sequence of conditions and assays) and ultimately in closed-loop
  wet-lab deployment is the necessary next step.
\item \textbf{Low-dimensional decision spaces.} The optimization stages have one or two knobs and
  three candidate pathways (only two with modeled response surfaces). We hypothesize, but have not
  demonstrated, that diagnostics-first discrimination gains value over plate-splitting as the
  pathway count grows beyond what one plate can cover; Proposition~\ref{prop:dominance} bounds only
  the commit-first case.
\item \textbf{Hand-specified diagnostic models.} The \eC{} likelihood matrices for cheap diagnostics
  are specified from domain knowledge. In deployment these should be learned from historical assay
  data; robustness to likelihood misspecification is an open question.
\item \textbf{Two-stage policy, not a unified acquisition.} The implementation is an explicitly
  two-stage procedure; the discrimination and optimization utilities are not commensurable, and we
  make no claim of a single cross-stage acquisition. Defining a common-currency criterion (for
  example, expected reduction in remaining cost-to-target) is open.
\item \textbf{Theory--implementation gap.} The discrimination guarantee cited as motivation is for
  exact \eC{}; the implementation uses posterior-impurity selection with threshold stopping, for
  which no competitive ratio is claimed. The GP-UCB regret bound is for an idealized fixed-kernel
  sequential procedure, not the empirical-Bayes batched implementation.
\item \textbf{Baseline coverage.} The split-plate baseline matches the method's cost on this
  benchmark; mixed-space BO over (pathway, conditions), portfolio allocation across pathways, and a
  Bayes-optimal dynamic-programming policy on the small hypothesis set are not implemented and
  remain the strongest untested competitors.
\item \textbf{Center-biased initialization, not transfer.} The initial-design bias is fixed; a shared
  cross-campaign surrogate that accumulates structure across feedstocks would strengthen it into
  genuine transfer learning, with its own analysis.
\end{itemize}

\section{Future Work}\label{sec:future}
The natural progression is from surrogate benchmarks to instruments. First, \emph{replay}: given
campaign-level logs from an autonomous platform (sequences of conditions and assays, not only final
protocols), the policy can be evaluated counterfactually on real chemistry with no new experiments,
turning any archived campaign into a benchmark. Second, \emph{closed-loop validation}: running the
engine against a platform's incumbent optimizer on live campaigns, with matched budgets, in
selective precipitation and in structurally identical decision-then-optimize problems such as
electrolyte formulation, where discrete family choices precede continuous ratio tuning and cheap
screens precede expensive cycling tests. Third, \emph{scaling the decision space}: larger hypothesis
sets, staged and multi-fidelity assay hierarchies, learned diagnostic likelihoods, and
multi-objective targets, where explicit pricing should matter most. Finally,
\emph{integration}: agentic laboratory platforms already generate hypothesis spaces, protocols, and
economic screens with language-model agents; the framework here can serve as the budget-allocation
and experiment-selection component beneath such planners, providing explicit pathway uncertainty, cost
accounting, and experiment-selection logic. The implementation and benchmark released with this paper
are intended to support this line of work.

\paragraph{Hypothesis-expanding experiment design.} The present framework selects among a fixed set
of candidate recovery pathways. A more ambitious extension would use unexplained structure in
campaign measurements, particularly full multi-element ICP-MS outputs rather than a single scalar
objective, to propose additional pathways, co-product targets, staged precipitation sequences, or
discriminating experiments. With access to campaign-level data from a platform such as CICERO, an
initial study could proceed as: (i) retrospective reconstruction from raw per-well data;
(ii) blind prediction on a held-out plate; (iii) a jointly designed prospective discrimination
plate; (iv) dynamic expansion of the pathway hypothesis set when observations are poorly explained;
and (v) validation of any newly identified co-product or sequential separation pathway. This would
test whether active experiment selection enables campaigns that are otherwise infeasible because the
combined pathway, protocol, and measurement space is too large for exhaustive exploration. We
emphasize that this is future collaborative work, not a result of the present paper.

\section{Conclusion}
Autonomous campaigns must decide what to run, not only how to tune it. We formulated this as a
sequential decision problem that separates recovery-pathway identification from within-pathway
optimization under heterogeneous costs, gave a fixed-budget bound on expected campaign spend with an explicit success probability,
and evaluated a two-stage cost-sensitive policy on CICERO-inspired synthetic benchmarks. The policy
matches an oracle-pathway reference and a split-plate baseline without an oracle pathway label, and
avoids the penalty of a simulated commit-first baseline; whether priced diagnostics outperform
plate-based discrimination at larger pathway counts is the main open empirical question. As
self-driving laboratories automate not only execution but selection, explicit prices on actions and
explicit statements of expected spend seem likely to become standard components of campaign design.

{\small
\bibliographystyle{plainnat}
\bibliography{references}
}

\appendix
\section{Proofs}\label{app:proofs}

\paragraph{Theorem~\ref{thm:opt}.}
Fix a horizon $T$ in advance and run IGP-UCB \citep{chowdhury2017kernelized} with posterior
regularization $\lambda=1+2/T$ and $\gamma_t$ defined under that regularization. Under
Assumption~\ref{ass:opt}, Theorem~2 of \citet{chowdhury2017kernelized} supplies the anytime
confidence bound (their coefficient $\beta_t$ equals this paper's $\sqrt{\beta_t}$ in the
convention $\mu+\sqrt{\beta_t}\,\sigma$), and their Theorem~3 converts it into the IGP-UCB
cumulative regret bound: with probability $\ge1-\delta$,
$R_T=O\big(\sqrt{T}\,(B\sqrt{\gamma_T}+\gamma_T)\big)$. Since
$\min_{t\le T} r_t \le R_T/T$, some queried point is $\varepsilon$-optimal once
$R_T/T\le\varepsilon$; this establishes that such a point was \emph{queried}, not that it is
identified or returned from noisy observations. The sufficient condition $\beta_T\gamma_T=o(T)$,
which the squared-exponential kernel satisfies ($\gamma_T=O((\log T)^{d+1})$), ensures eventual
solvability of $R_T/T\le\varepsilon$ for any fixed $\varepsilon$; the condition is sufficient,
not necessary. Plate counts in
Prop.~\ref{thm:composite} follow from its fixed-budget assumption, not from a batch-regret theorem.
$\square$

\paragraph{Proposition~\ref{thm:composite}.} Proved in the main text (linearity of
expectation over the two sequential stage spends; union bound over the two failure events).

\paragraph{Proposition~\ref{prop:dominance}.}
With probability $1-p_{\mathrm{corr}}$ the commit-first baseline selects an incorrect region and,
before it can switch, expends at least $T_{\mathrm{fail}}$ evaluations
($\lceil T_{\mathrm{fail}}/q\rceil$ plates at $c_{\mathrm{opt}}$ each); with probability
$p_{\mathrm{corr}}$ it incurs no penalty. Its expected spend therefore exceeds the within-region
optimization spend by at least $(1-p_{\mathrm{corr}})c_{\mathrm{opt}}\lceil T_{\mathrm{fail}}/q\rceil$.
The two-stage procedure pays expected discrimination cost $C_{\mathrm{disc}}$. The residual
optimization spends $C_{\mathrm{post}}^{\mathrm{CF}}$ and $C_{\mathrm{post}}^{\mathrm{TS}}$ are
defined on every branch and exclude the wrong-commitment term already counted above, so no cost is
double-counted. $\E[C_{\mathrm{post}}^{\mathrm{CF}}-C_{\mathrm{post}}^{\mathrm{TS}}]\ge 0$ by
the stated expectation assumption; retaining this term and subtracting gives the bound. Under a prior whose induced
distribution over decision regions is uniform, $p_{\mathrm{corr}}=1/|\Reg|$ exactly, so
$1-p_{\mathrm{corr}}=1-1/|\Reg|$ (bounded above by one). $\square$

\section{Background: guarantees for adaptive-submodular cover}\label{app:bounds}
We summarize this literature by theorem scope; none of these guarantees applies to the implemented
posterior-impurity objective, and quantities $Q,\eta$ are not transferable between differently
normalized objectives. (i) Exact \eC{} on the finite deterministic-realization expansion retains
adaptive submodularity under correlated noise \citep{golovin2010near}; its original logarithmic
approximation proof, however, invoked the general adaptive-cover theorem later shown to be flawed
\citep{nan2017comments}. (ii) \citet{hellerstein2021tight} restore the $\ln(Q/\eta)+1$ bound for
stochastic submodular cover under \emph{independent item states}; it does not apply outside that
setting, and outcomes here are correlated through the latent hypothesis. Correlation excludes this
theorem specifically, not all adaptive-cover results; for our implementation it is instead the
posterior-impurity surrogate, the threshold target, and the unestablished coverability hypotheses
that prevent invoking the correlated-realization guarantees. (iii) For correlated realizations,
\citet{esfandiari2021adaptivity} prove a unit-cost bound
$c_{\mathrm{greedy}}\le(c_{\mathrm{opt}}+1)\ln(nQ/\eta)+1$ under adaptive monotonicity, adaptive
submodularity, an $\eta$-gap, and pointwise coverability; for nonuniform costs,
\citet{cui2023greedy} prove a $4(1+\ln(Q/\eta))$ bound in their observation-based monotone
coverable model. (iv) \citet{harris2024lower} show the corrected squared-log greedy theorem for
general adaptive cover is false even at normalized quota $Q=1$; this is a statement about general
adaptive cover, not about \eC{} specifically. (v) Exact identification of the true region under the
noisy expansion additionally requires every complete outcome vector to be compatible with a single
region, which we do not assume; \eC{} instead determines the equivalence class of the
decision-after-all-tests, a different target from our posterior-threshold identification.
\textsc{eced} \citep{chen2017near} analyzes a different error-based objective under its own noise
assumptions.

\section{Benchmark and Implementation Details}\label{app:sim}
\paragraph{Constrained endpoints.} The synthetic surfaces reproduce selected reported endpoints:
produced-water Mg(OH)$_2$ purity
$\ge98.5\%$ near $3.4$ NaOH equivalents with $\approx\!86\%$ yield; SmCo Sm purity peak $\approx\!92\%$;
NdFeB oxalate Fe/REE separation factor $\approx\!200$ (median over the concentration range), hydroxide
$<30$ at high concentration, and Nd/Pr separation factor up to $\approx\!1.34$. A test suite asserts
each of these. Assay noise is Gaussian with relative standard deviation $4.5\%$; this is an author
assumption motivated by the $4.0$--$7.6\%$ repeatability range CICERO reports for one produced-water
experiment, not a measured noise model for all assays. Purity observations are clipped to
$[0,100]\%$, which can bias simulated observations near the optimum; alternative noise forms
(additive, heteroscedastic, log-normal for separation factors) are not evaluated here.

\paragraph{Provenance of assumptions.} Quantities used by the benchmark and their sources:
\begin{center}\small
\begin{tabular}{lll}
\toprule
Quantity & Value & Source \\
\midrule
Plate structure (all instances) & $q{=}8$, $12$ replicates & CICERO produced-water expt.; imposed elsewhere \\
Per-well cost & \$50 & author assumption \\
Diagnostic prices & \$60--\$120 & author assumption \\
Hypothesis set & $3$ per instance & author construction \\
Diagnostic likelihoods & $0.70$--$0.85$ on-diagonal & author construction \\
Posterior stopping threshold & $0.9$ & author choice \\
Assay noise & $4.5\%$ relative Gaussian & assumption motivated by $4.0$--$7.6\%$ range \\
Commit-first failure horizon & $4$ plates before backtrack & author construction \\
Response-surface shapes & closed forms below & author-designed, endpoint-constrained \\
\bottomrule
\end{tabular}
\end{center}
Absolute dollar figures are illustrative; ratios between methods under a fixed cost model are the
meaningful quantities, and plate-equivalent units (one plate $=\$4{,}800$) can be substituted.
\paragraph{Methods and costs.} Plates have $q{=}8$ conditions $\times\,12$ replicates ($96$ wells) at
\$50/well. Diagnostics cost \$60--\$120. GP: ARD squared-exponential kernel, hyperparameters fit by
marginal-likelihood maximization; GP-UCB schedule as in \S\ref{sec:opt}; batches by kriging-believer
greedy. The Vanilla-BO baseline uses the identical GP and UCB but no discovery, no cost-normalization,
and a cold (space-filling) first plate; on NdFeB it commits to the cheaper default (hydroxide) and
backtracks on failure. All results average $20$ seeds; the reference implementation is pure NumPy.
\paragraph{Optimizer quality vs.\ random (isolated).} To confirm the GP-UCB baseline is a competent
optimizer rather than a strawman, we run an isolated study on the SmCo surface (no discovery, no
warm-start, the same initial plate for both methods), measuring the best \emph{true} value reached and
plates-to-target as the target hardens ($60$ seeds; script \texttt{optimizer\_vs\_random.py}):

\begin{center}\small
\begin{tabular}{lccccc}
\toprule
best true purity by round & 1 & 2 & 3 & 4 & 5 \\
\midrule
GP-UCB & 82.1 & 88.9 & 91.0 & 91.9 & 91.9 \\
Random & 82.1 & 86.3 & 87.6 & 88.6 & 89.0 \\
\bottomrule
\end{tabular}\qquad
\begin{tabular}{lcc}
\toprule
plates to target & GP-UCB & Random (reach) \\
\midrule
$86\%$   & 1.93 & 2.67 (95\%) \\
$90\%$   & 2.40 & 3.47 (63\%) \\
$91.5\%$ & 2.77 & 3.74 (45\%) \\
\bottomrule
\end{tabular}
\end{center}
GP-UCB is consistently and increasingly better than random as the target tightens (true peak $92.0\%$);
the tie on cost-to-target in the main benchmark is therefore a property of the easy targets, not the
optimizer.

\paragraph{Reproducibility.} Code and benchmark:
\url{https://github.com/coactivescience/cost-aware-active-learning} (MIT license; an archived
release with a DOI will accompany the public preprint). Seeds $0$--$19$;
\texttt{python tests/test\_coactive.py} checks the endpoint constraints and the
main empirical claims (Coactive learning is $\ge3\times$ cheaper than the default agent on NdFeB on every seed,
is within $1.4\times$ of Oracle-pathway BO's cost, and reaches $\ge93\%$ of the achievable optimum).
All tables, figures, and the study above are regenerated by \texttt{run\_benchmark.py},
\texttt{make\_figures.py}, and \texttt{optimizer\_vs\_random.py}.

\paragraph{Response-surface equations.} Let $\varsigma(z)=(1+e^{-z})^{-1}$. The synthetic ground-truth
surfaces (in \%, or separation factor where noted) are:
\begin{align*}
\text{Produced water (Mg):}\quad & g(\text{eq})=\min\!\big(99.4-\kappa(\text{eq})(\text{eq}-3.4)^2,\ 99.4-58.4\,\varsigma(\tfrac{\text{eq}-5.5}{0.45})\big),\\
 & \kappa=0.8\ (\text{eq}\le3.4),\ 2.0\ (\text{eq}>3.4);\quad \text{yield}=40+58\,\varsigma(\tfrac{\text{eq}-2.4}{0.7}).\\
\text{SmCo (Sm):}\quad & g(b,h)=34+58\exp\!\big(-\tfrac12[((b-2.1)/0.62)^2+((h-1.0)/0.85)^2]\big).\\
\text{NdFeB Fe/REE:}\quad & s_{\text{ox}}(c)=150+70\,\varsigma(\tfrac{c-0.6}{0.5}),\quad
 s_{\text{OH}}(c)=60\,e^{-\frac12((c-0.4)/0.5)^2}+8.\\
\text{NdFeB Nd/Pr:}\quad & s(o,\mathrm{pH})=\big(1+0.35\,\varsigma(\tfrac{o-1.6}{0.38})\big)\big(1+0.015(\mathrm{pH}-3)\big).
\end{align*}
Each assay returns the surface value plus Gaussian noise of relative standard deviation $0.045$
(within CICERO's $4.0$--$7.6\%$ band); purity is clipped to $[0,100]\%$.

\paragraph{EC\textsuperscript{2} instances.} Each feedstock defines $3$ hypotheses (a correct
target$\times$reagent and two competitors/decoys), with priced diagnostics (\$60--\$120) whose
row-stochastic likelihood matrices place $0.70$--$0.85$ mass on the outcome consistent with the true
hypothesis. The cost-sensitive greedy \eqref{eq:ec2greedy} runs diagnostics until one decision region
exceeds posterior mass $0.9$; for NdFeB this commits oxalate after two assays (\$185 total;
Fig.~\ref{fig:ec2}).

\paragraph{Hyperparameters.} GP: zero-mean prior, ARD squared-exponential kernel; lengthscales,
signal variance, and noise fit by marginal-likelihood maximization with $8$ random restarts and
coordinate refinement. GP-UCB schedule $\beta_t=2\log(|D|t^2\pi^2/6\delta)$ with $\delta=0.1$ and
$|D|=800$ Latin-hypercube candidates. Plates use $q=8$ conditions selected by kriging-believer greedy.
Warm-starting mixes the first space-filling plate with a neutral central prior
($x \leftarrow 0.6\,x_{\mathrm{LHS}} + 0.4\,\bar x$ with $\bar x$ the domain centre), with no
knowledge of the optimum; the baseline omits this and uses a pure space-filling first plate.

\end{document}